\documentclass[nohyperref]{article}

\usepackage{microtype}
\usepackage{graphicx}
\usepackage{subfigure}
\usepackage{dsfont}
\usepackage{booktabs} % for professional tables

\usepackage{hyperref}

\usepackage[accepted]{icml2022}

\usepackage{amsmath}
\usepackage{amssymb}
\usepackage{mathtools}
\usepackage{amsthm}
\usepackage{thmtools}

\usepackage[capitalize,noabbrev]{cleveref}

\theoremstyle{plain}
\newtheorem{theorem}{Theorem}[section]
\newtheorem{proposition}[theorem]{Proposition}
\newtheorem{lemma}[theorem]{Lemma}

\theoremstyle{definition}

\theoremstyle{remark}

\usepackage[textsize=tiny]{todonotes}

\def\KL{\mathbf{d}_{\mathrm{KL}}}
\def\CE{\mathbf{d}_{\mathrm{CE}}}
\def\environment{\mathcal{E}}

\def\proxy{\tilde{\environment}}
\def\data{\mathcal{D}}

\def\E{\mathbb{E}}
\def\I{\mathbb{I}}
\def\H{\mathbb{H}}
\def\Regret{{\rm Regret}}

\newcommand{\Yb}{\mathbf{Y}}

\newcommand{\Real}{\mathds{R}}

\newcommand{\Prob}{\mathds{P}}

\DeclareMathOperator*{\argmax}{arg\,max}

\icmltitlerunning{From Predictions to Decisions: the Importance of Joint Predictive Distributions}

\begin{document}

\twocolumn[
\icmltitle{
From Predictions to Decisions: \\The Importance of Joint Predictive Distributions
}



\icmlsetsymbol{equal}{*}

\begin{icmlauthorlist}
\icmlauthor{Zheng Wen}{yyy}
\icmlauthor{Ian Osband}{yyy}
\icmlauthor{Chao Qin}{sch}
\icmlauthor{Xiuyuan Lu}{yyy}
\icmlauthor{Morteza Ibrahimi}{yyy}
\icmlauthor{Vikranth Dwaracherla}{yyy}
\icmlauthor{Mohammad Asghari}{yyy}
\icmlauthor{Benjamin Van Roy}{yyy}
\end{icmlauthorlist}

\icmlaffiliation{yyy}{DeepMind, Mountain View, CA 94043}
\icmlaffiliation{sch}{Columbia Business School, Columbia University, New York, NY 10027}

\icmlcorrespondingauthor{Zheng Wen}{zhengwen@google.com}

\icmlkeywords{Machine Learning, ICML}

\vskip 0.3in
]



\printAffiliationsAndNotice{}  % leave blank if no need to mention equal contribution

\begin{abstract}
    A fundamental challenge for any intelligent system is prediction: given some inputs, can you predict corresponding outcomes? Most work on supervised learning has focused on producing accurate marginal predictions for each input. However, we show that for a broad class of decision problems, accurate joint predictions are required to deliver good performance. In particular, we establish several results pertaining to combinatorial decision problems, sequential predictions, and multi-armed bandits to elucidate the essential role of joint predictive distributions. Our treatment of multi-armed bandits introduces an approximate Thompson sampling algorithm and analytic techniques that lead to a new kind of regret bound.
\end{abstract}

\section{Introduction}

We are motivated by the problem of an intelligent agent interacting with an unknown environment.
Although the agent is initially uncertain of the true dynamics of the environment, it might learn to predict aspects of its evolution through observation and learning.  Predictions can then guide decisions aimed to optimize future outcomes.

Most work on supervised learning has focused on producing accurate marginal predictions for each input. For instance, in the image classification problem, the goal is to predict the class label of each image correctly. Though accurate marginal predictions are sufficient for some applications, in this paper, we show that for a broad class of downstream tasks, accurate joint predictions are essential. In particular, we show that even in simple combinatorial and sequential decision problems, accurate marginal predictions are insufficient for effective decisions. Rather, accurate joint predictive distributions are required to achieve good performance. We also show that, in the multi-armed bandit, a classical sequential decision problem, accurate joint predictive distributions enable near-optimal performance.

The cross-entropy loss is perhaps the most widely used performance metric in machine learning and plays a central role in training deep learning systems \cite{lecun2015deep}.  It is typically used to evaluate the quality of marginal predictions but readily extends to address joint predictive distributions. Baselining the cross-entropy loss with an agent-independent constant gives rise to expected KL-divergence. Hence, the cross-entropy loss is equivalent to the expected KL-divergence as a metric for ranking agents.  We will focus in this paper on the use of the expected KL-divergence because the baselining allows for more elegant analysis and framing of results.  In particular, the KL-divergence is closely related to fundamental notions in information theory \cite{cover1999elements}, such as entropy and mutual information. It is worth clarifying that the cross-entropy loss and expected KL-divergence are just particular choices of evaluation metrics.  The importance of evaluating joint predictions, not just marginals, should similarly be apparent under other evaluation metrics.

To elucidate the essential role of joint predictive distributions in driving downstream decision tasks, we establish several results pertaining to the universality of our metric across tasks, combinatorial decision problems, sequential predictions, and mult-armed bandits.  Among these, our study of multi-armed bandits is particularly innovative, entailing the development of a novel variant of the Thompson sampling (TS) algorithm \cite{Thompson1933, russo2017tutorial} and original analytic techniques that lead to a new kind of regret bound (Theorem~\ref{thm:regret_bound}).  Our proposed approximate TS algorithm is a generalization of the standard (exact) TS algorithm and is general in the sense that it does not depend on the agent's uncertainty representation. 
Instead, it only requires that the agent can simulate hypothetical observations, sampled from a joint predictive distribution. The new analytic techniques required to establish our regret bound build on recent progress in information-theoretic analysis of bandit and reinforcement learning algorithms \cite{russo2016information, lu2021reinforcement}.  We believe that the key ideas can also offer new insight into approximate TS algorithms beyond the scope of this paper, including those designed for more general sequential decision problems, such as episodic reinforcement learning.

Also novel is our analysis of sequential prediction.  In this case, we establish a result that highlights the importance of joint predictive distributions even when an agent aims only to produce accurate marginal predictions sequentially, as it streams data.  In particular, for the agent to be able to incrementally update its knowledge representation as it streams, it must also be capable of generating accurate joint, not just marginal, predictive distributions.

The remainder of the paper is organized as follows: in Section~\ref{sec:evaluation_metric}, we introduce the notation for this paper via a standard supervised learning setting.  We also review and discuss the evaluation metric for joint predictive distributions. In Section~\ref{sec:combinatorial}, we illustrate through a simple recommendation system example that accurate marginal predictions are insufficient to guide effective decisions in combinatorial decision problems, and that accurate joint predictive distributions are essential. In Section~\ref{sec:sequential_prediction}, we formally prove that in a sequential decision problem, an agent that updates its parameters incrementally must retain sufficient information about joint predictive distributions if it is to achieve good performance, even if the performance metric only depends on marginal distributions. In Section~\ref{sec:mab}, we consider the classical multi-armed bandit and show that accurate marginal predictions are insufficient to guide good decisions. We also establish a regret bound for our proposed TS algorithm: if the agent produces accurate joint predictive distributions, then the TS algorithm can attain near-optimal regret. Finally, we will discuss the related work in Section~\ref{sec:related_work} and conclude the paper in Section~\ref{sec:conclusion}.

\section{Evaluating joint predictive distributions}
\label{sec:evaluation_metric}

In this section, we introduce the notation for this paper via a standard supervised learning framework. We also motivate and discuss the evaluation metrics used in this paper for joint predictive distributions.

\subsection{Supervised learning}
\label{sec:supervised_learning}
Consider a sequence of pairs $((X_t, Y_{t+1}): t =0,1,2,\ldots)$, where each $X_t$ is a feature vector and each $Y_{t+1}$ is its target label. Feature vectors $(X_t: t =0,1,2,\ldots)$ are i.i.d. according to an input distribution $P_X$.  Each target label $Y_{t+1}$ is independent of all other data, conditioned on $X_t$, and distributed according to $\environment (\cdot | X_t)$.  The conditional distribution $\environment$ is referred to as the \emph{environment}. The environment $\environment$ is random; and this reflects the agent's uncertainty about how labels are generated given features. Note that $\Prob(Y_{t+1} \in \cdot | \environment, X_t) = \environment(\cdot|X_t)$ and $\Prob(Y_{t+1} \in \cdot | X_t) = \E[\environment(\cdot|X_t) | X_t]$.

In supervised learning, we consider an agent that learns about the environment $\environment$ from a training dataset 
\[
\data_T \equiv ((X_t, Y_{t+1}): t =0,1,\ldots, T-1),
\]
and aims to predict the target labels 
\[
Y_{T+1:T+\tau} \equiv (Y_{T+1}, \dots, Y_{T+\tau})
\]
at $\tau$ feature vectors $X_{T:T+\tau-1} \equiv (X_T, \dots, X_{T+\tau-1})$.

Conditioned on the environment $\environment$, a predictive distribution over the target labels is given by
$$P^*_{T+1:T+\tau} \equiv \Prob \left(Y_{T+1:T+\tau} \in \cdot \middle | \environment , X_{T:T+\tau-1} \right).$$
Conditioned instead on the training data, the predictive distribution becomes
$$\overline{P}_{T+1:T+\tau} \equiv \Prob \left(Y_{T+1:T+\tau} \in \cdot \middle | \data_T , X_{T:T+\tau-1} \right).$$
In a sense, $\overline{P}_{T+1:T+\tau}$ represents the result of perfect inference.  Since the associated computations are intractable for environments of interest, we instead consider agents that approximate this predictive distribution. Specifically, we consider agents that represent the approximation in terms of a generative model.  The agent's predictions are parameterized by a vector $\theta_T$ that the agent learns from the training data $\data_T$.  The vector $\theta_T$ is conditionally independent of $\environment$ conditioned on $\data_T$.  For any inputs $X_{T:T+\tau-1}$, $\theta_T$ determines a predictive distribution, which could be used to sample imagined outcomes $\hat{Y}_{T+1:T+\tau}$.  Hence, the agent's $\tau^{\rm th}$-order predictive distribution is given by
$$\hat{P}_{T+1:T+\tau} \equiv \Prob(\hat{Y}_{T+1:T+\tau} \in \cdot | \theta_T, X_{T:T+\tau-1}).$$
When $\tau=1$, we alternatively use $\hat{P}_{T+1}$, $\overline{P}_{T+1}$, and $P^*_{T+1}$ to denote $\hat{P}_{T+1:T+\tau}$, $\overline{P}_{T+1:T+\tau}$, and $P^*_{T+1:T+\tau}$, respectively.

Note that if $\tau=1$, $\hat{P}_{T+1}$ represents a marginal prediction: it predicts the label $Y_{T+1}$ for a single input $X_T$. On the other hand, for $\tau>1$, it represents a joint prediction over labels at $\tau$ input features.

\subsection{Marginal vs. joint predictive distributions}
\label{sec:example_coin}

Before proceeding, we use a simple coin flipping example to illustrate that evaluating marginal and joint predictive distributions can result in very different answers. Suppose that
$(Y_{t+1}: t=0,1,\ldots)$ are generated by repeated tosses of a possibly biased coin with unknown probability $p$ of heads, with $Y_{t+1}=1$ and $Y_{t+1}=0$ indicating heads and tails, respectively.  Consider two agents with different beliefs: Agent 1 assumes $p = 2/3$ and models the outcome of each coin toss as independent conditioned on $p$.
Agent 2 assumes that $p = 1$ with probability $2/3$ and $p=0$ with probability $1/3$; that is, the coin either produces only heads or only tails. Let $\hat{Y}^1_{t+1}$ and $\hat{Y}^2_{t+1}$ denote the outcomes imagined by the two agents.
Despite their differing assumptions, the two agents generate identical marginal predictive distributions: $\Prob(\hat{Y}^1_{t+1}=0) = \Prob(\hat{Y}^2_{t+1}=0) = 1/3$.

On the other hand, the joint predictions of these two agents differ for $\tau>1$:
$$\Prob(\hat{Y}^1_1, \dots, \hat{Y}^1_\tau=0) = 1/3^\tau < 1/3 = \Prob(\hat{Y}^2_1, \dots, \hat{Y}^2_\tau = 0).$$
Evaluating marginal predictions cannot distinguish between the two agents, though for a specific prior distribution over $p$, one agent could be right and the other wrong.
One must evaluate joint predictions to make this distinction.

\subsection{Cross-entropy loss}

Cross-entropy loss is perhaps the most widely used metric in machine learning, though it is typically only used to evaluate marginal predictive distributions. For our supervised learning formulation, an agent's marginal cross-entropy loss takes the form
\begin{equation}
\label{eq:cross_entropy_marginal}
\CE^1 \equiv - \E \big[\log \hat{P}_{T+1}(Y_{T+1}) \big],
\end{equation}
where the expectation is over both $\hat{P}_{T+1}$ and $Y_{T+1}$, and the superscript ``$1$" in $\CE^1$ indicates that this evaluates marginal predictions. Note that the marginal distribution $\hat{P}_{T+1}$ is random because it depends on $\theta_T$ and $X_T$.

It is straightforward to extend the cross-entropy loss to assess joint predictive distributions. Generalizing Equation~\eqref{eq:cross_entropy_marginal}, for any $\tau=1,2,\ldots$, we define the $\tau^{\textrm{th}}$-order cross-entropy loss:
\begin{equation}
\label{eq:cross_entropy_joint}
\CE^\tau \equiv - \E \big[\log \hat{P}_{T+1:T+\tau}(Y_{T+1:T+\tau}) \big],
\end{equation}
where the expectation is over $\hat{P}_{T+1:T+\tau}$
and $Y_{T+1:T+\tau}$. Note that the $\tau^{\textrm{th}}$-order joint distribution $\hat{P}_{T+1:T+\tau}$ is also random, since it depends on $\theta_T$ and $X_{T:T+\tau-1}$.

\subsection{Kullback–Leibler divergence}

One can use the joint cross-entropy loss defined in Equation~\eqref{eq:cross_entropy_joint} to assess joint predictive distributions. However, it can be helpful to offset the metric by a baseline to convert it into the Kullback–Leibler (KL) divergence.  We will take $\overline{P}_{T+1:T+\tau}$, which is the perfect posterior predictive, to be our baseline.  As we will see, since $\overline{P}_{T+1:T+\tau}$ does not depend on the agent, our measure of KL-divergence and the cross-entropy loss are effectively equivalent in the sense that they only differ by a constant that does not depend on the agent. However, KL-divergence exhibits properties that allow for a more elegant mathematical analysis.

The $\tau^{\textrm{th}}$-order expected KL-divergence with respect to $\overline{P}$ is defined by
\begin{equation}
\label{eq:kl_tau}
\KL^\tau \equiv \E \big [ \KL \big( \overline{P}_{T+1:T+\tau}
\big \| \hat{P}_{T+1:T+\tau} \big )\big],
\end{equation}
where the expectation is over the distributions $\overline{P}_{T+1:T+\tau}$ and $\hat{P}_{T+1:T+\tau}$, which depend in turn on the data $\data_T$, the agent parameters $\theta_T$, and the $\tau$ inputs $X_{T:T+\tau-1}$.

Note that KL-divergence is minimized when $\hat{P}_{T+1:T+\tau} = \overline{P}_{T+1:T+\tau}$, with the minimum being zero.  Further, our two metrics are related according to
$$\KL^\tau = \CE^\tau + \E[\log \overline{P}_{T+1:T+\tau}(Y_{T+1:T+\tau})].$$
Since $\overline{P}_{T+1:T+\tau}(Y_{T+1:T+\tau})$ does not depend on the agent, our two metrics rank agents identically. 

An unbiased estimate of cross-entropy loss can be computed based on a test data sample, according to
$$\CE^\tau \approx - \log \hat{P}_{T+1:T+\tau}(Y_{T+1:T+\tau}) \big.$$
The same is not true for $\KL^\tau$, which can only be estimated if also given an estimate of $\E[\log \overline{P}_{T+1:T+\tau}(Y_{T+1:T+\tau})]$.  Hence, $\KL^\tau$ serves only as conceptual tools in our analysis and not an evaluation metric that can be applied with empirical data.  While it ranks agents identically with $\CE^\tau$, $\KL^\tau$ is more natural as a metric since its minimum is zero and it accommodates more elegant analysis.

\subsection{Error in predictions versus environment}

Our $\KL^\tau$ metric assesses error incurred by the predictive distribution $\hat{P}_{T+1:T+\tau}$.  A common approach to generating such a predictive distribution is through estimating a posterior distribution over environments and using that posterior distribution to generate the predictive distribution.  In such a context, $\theta_T$ parameterizes the estimated posterior distribution.  Let $\hat{\environment}$ be an imaginary environment sampled from this posterior distribution.  To offer some intuition for $\KL^\tau$, we consider in this section its relation to the KL-divergence between the distributions of the true and imaginary environments.

Let $\hat{Y}_{T+1:T+\tau}$ denote a sequence of imaginary outcomes, with each $\hat{Y}_{t+1}$ sampled independently from $\hat{\environment}(\cdot | X_t)$.  If the support of the input distribution $P_X$ is exhaustive, the support of the imaginary environment distribution $\Prob(\hat{\environment} \in \cdot | \theta_T)$ contains that of the true environment distribution $\Prob(\environment \in \cdot | \data_T)$, and the environment distributions satisfy suitable regularity conditions, then
$$\lim_{\tau \rightarrow \infty} \KL^\tau = \E \big [\KL \big( \Prob(\environment \in \cdot | \data_T )
\big \|
\Prob(\hat{\environment} \in \cdot | \theta_T )
\big) \big].$$
In other words, under certain technical conditions, as the number $\tau$ of test data pairs grows, $\KL^\tau$ converges to the error in the estimated posterior distribution over environments, measured in terms of KL-divergence.

One might wonder why we should use $\KL^\tau$ rather than the KL divergence between the true and imaginary environments
\begin{equation}
\label{eq:kl_env}
\E \big [\KL \big( \Prob(\environment \in \cdot | \data_T )
\big \|
\Prob(\hat{\environment} \in \cdot | \theta_T )
\big) \big]
\end{equation}
to evaluate the agents.  There are several reasons for it. First, practical agent designs often do not satisfy the requisite regularity conditions, and hence \eqref{eq:kl_env} becomes infinite.  For example, it is common to approximate the posterior distribution $\environment$ using an ensemble of environment models (see, e.g., \citet{lu2017ensemble}).  Such an ensemble represents a distribution with finite support though the posterior may have infinite support.  On the other hand, for any finite $\tau$, $\KL^\tau$ is finite.  Second, $\CE^\tau$, which is equivalent to $\KL^\tau$ up to a constant, can be computed based on data, whereas computing \eqref{eq:kl_env} requires access to the posterior distribution of $\environment$.  Finally, as we will establish later, $\KL^\tau$ with finite $\tau$ is sufficient to support effective decisions in downstream tasks such as multi-armed bandits.

\subsection{Universality of $\KL^\tau$}
\label{sec:universal_kl}

We now show that, for any $\tau$, accuracy in terms of $\KL^\tau$ is sufficient to guarantee an effective decision if the decision is judged in relation only to $Y_{T+1:T+\tau}$.  In particular, suppose an action $a$ selected from a set $\mathcal{A}$ results in an expected reward 
\begin{align*}
&\E[r(a,Y_{T+1:T+\tau}) | \data_T, X_{T:T+\tau-1}] \\
&= \sum_{y_{T+1:T+\tau}} \overline{P}_{T+1:T+\tau}(y_{T+1:T+\tau}) r(a, y_{T+1:T+\tau}),
\end{align*}
where $r$ is a reward function with range $[0,1]$.  The following result bounds the loss in expected reward of a decision that is based on the estimate $\hat{P}_{T+1:T+\tau}$ instead of the posterior $\overline{P}_{T+1:T+\tau}$.
\begin{proposition}
\label{prop:universality}
If an action $\hat{a} \in \mathcal{A}$ maximizes 
$$\sum_{y_{T+1:T+\tau}} \hat{P}_{T+1:T+\tau}(y_{T+1:T+\tau}) r(a, y_{T+1:T+\tau})$$
then
\begin{align*}
\E[r(\hat{a}, Y_{T+1:T+\tau})] \geq \max_{a\in \mathcal{A}} \E[r(a, Y_{T+1:T+\tau})] - \sqrt{2 \KL^\tau}.
\end{align*}
\end{proposition}
This results follows from Pinsker's inequality and Jensen's inequality. A proof is provided in  Appendix~\ref{app:proof_for_universality}.

Proposition~\ref{prop:universality} ensures that if a decision maximizes an approximation to expected reward, with the expectation approximated using the predictive distribution $\hat{P}_{T+1:T+\tau}$, then the action will be within $\sqrt{2 \KL^\tau}$ of what is achievable given the posterior predictive distribution $\overline{P}_{T+1:T+\tau}$.  In this sense, $\KL^\tau$ is a universal evaluation metric: its value ensures a level of performance in any decision problem.

\section{Combinatorial decision problems}
\label{sec:combinatorial}

In this section, we will present a simple example that highlights the importance of the joint predictive distributions in combinatorial decision problems.
Importantly, we will show that there can be benefit to considering a joint predictive distribution even when the data are i.i.d. conditioned on the environment.

Consider the problem of a customer interacting with a recommendation system that proposes a selection of $K>1$ movies from an inventory of $N$ movies $X_1,..,X_N$.
Each $X_i \in \Real^d$ describes the features of movie $i$, and $d$ is the feature dimension. We model the probability that a user will enjoy movie $i$ by a logistic model $Y_i \sim {\rm logit}(\phi_*^T X_i)$, where ${\rm logit}$ is the standard logistic function.
Note that $\phi_* \in \Real^d$ describes the preferences of the user, which is not fully known to the recommendation system and can be viewed as a random variable. We will show that, in order to maximize the probability that the user enjoys at least one of the $K > 1$ recommended movies, it is insufficient to examine the marginal predictive distributions.

To make this example concrete, we consider the case where the user $\phi_*$ is drawn from two possible user types $\{ \phi_1, \phi_2 \}$, and the recommendation system should propose $K=2$ movies from an inventory $\{X_1, X_2, X_3, X_4\}$.
Table \ref{table:logits} presents the numerical values for each of these vectors in this problem and their associated probabilities of selection implied by their logit functions correct to two decimal places.
These values are chosen to set up a tension between optimization over marginal (each $X_i$ individually) and joint (pairs of $X_i, X_j$) predictions.
An agent that optimizes the expected probability for each movie individually will end up recommending the pair $(X_3, X_4)$ to an unknown $\phi \sim {\rm Unif}(\phi_1, \phi_2)$.
This means that, whether a user is type $\phi_1$ or $\phi_2$, there is a greater than 10\% chance they do not like either movie.
By contrast, an agent that considers the joint predictive distribution for $\tau \ge K=2$ can see that instead selecting the pair $(X_1, X_2)$ will give close to 100\% certainty that the user will enjoy one of the movies.

\begin{table*}%[!ht]
\begin{center}
\begin{tabular}{r|c|c|c|c}
  & $X_1 = (10, -10)$ & $X_2 = (-10, 10)$ & $X_3 = (1, 0)$ & $X_4 = (0, 1)$ \\
\hline
$\phi_1 = (1, 0)$ & 1 & 0 & 0.73 & 0.5 \\
$\phi_2 = (0, 1)$ & 0 & 1 & 0.5 & 0.73 \\
\hline
$\phi \sim {\rm Unif}(\phi_1, \phi_2)$ & 0.5 & 0.5 & 0.62 & 0.62
\end{tabular}
\end{center}
\caption{Expected probability to watch a movie under different user features, correct to two decimal places.}
\label{table:logits}
\end{table*}

This didactic example is designed to be maximally simplistic, and in some sense it labours an obvious point. In combinatorial decision problems where the outcome depends on the joint predictive distribution, optimization based on the marginal predictive distribution is insufficient to guarantee good decisions. By contrast, as shown in Section \ref{sec:universal_kl}, if your decision depends only on the $\tau^{\rm th}$-order predictive, then attaining small $\KL^\tau$ is sufficient for good performance.
What is perhaps interesting about this example however, is that this coupling can occur even where, conditioned on the true environment, the data generating process is actually i.i.d.
The key point is that, when the true underlying environment is unknown, a coupling in future observations is introduced.
As such, examining the joint predictive distribution can be essential for good performance.
In the rest of this paper, we will show that these distinctions become even more important as we move towards sequential decision problems.

\section{Sequential predictions}
\label{sec:sequential_prediction}

In this and next section, we will establish some theoretical results to justify the importance of joint predictive distributions in sequential decision problems.
We first consider sequential prediction problems, which is a subclass of sequential decision problems where the agent's predictions (actions) do not influence its future observations. In particular, we show that in a standard sequential prediction problem, for agents with incremental updates to achieve good performance, it is necessary for them to retain significant portion of information about joint predictive distributions. This is true even if the performance metric only depends on marginal distributions.

The considered sequential prediction problem is formulated as follows: data pairs $(X_t, Y_{t+1})$ arrive sequentially, one at a time. At each time $t$, the agent needs to compute parameters $\theta_t$ based on previously observed data pairs $\data_t = (X_0, Y_1, X_1, \ldots, X_{t-1}, Y_t)$. Then, a new data pair $(X_t, Y_{t+1})$ arrives. We assume that the feature vector $X_t$'s are unconditionally independent, but not necessarily identically distributed. The target label $Y_{t+1}$ is conditionally independently sampled from the distribution
$\environment \left( \cdot \middle | X_t \right )$, where $\environment$ is the environment. The agent's objective is to minimize the expected cumulative KL-divergence in the first $T$ time steps:
\begin{equation}
   \label{eq:cum_kl}
    \textstyle \sum_{t=0}^{T-1} \E \big [ \KL \big (\overline{P}_{t+1} \big \| \hat{P}_{t+1} \big ) \big], 
\end{equation}
where 
\begin{align}
    \bar{P}_{t+1} = & \, \Prob (Y_{t+1} \in \cdot | \, \data_t, X_t) \nonumber \\
    \hat{P}_{t+1} =& \, \Prob (\hat{Y}_{t+1} \in \cdot | \, \theta_t , X_t) \nonumber
\end{align}
for all time $t$. Note that this cumulative KL-divergence \eqref{eq:cum_kl} only depends on the marginal distributions $\overline{P}_{t+1}$ and $\hat{P}_{t+1}$. Also note that this performance metric is $0$ if the agent predicts the exact posterior at each time $t$.

We consider a setting where an agent needs to incrementally update its parameters as data arrive. Specifically, 
at time $t=0$, the agent chooses its parameters $\theta_0$ based on its prior knowledge; and then at each time $t=0,1,\ldots$, the agent updates its parameters incrementally by sampling from a distribution that only depends on $\theta_t$, $(X_t, Y_{t+1})$, and $t$:
\begin{equation}
\theta_{t+1} \sim \Prob \left(\theta_{t+1} \in \cdot \middle | \theta_t, X_t, Y_{t+1}, t \right). \label{eq:incremental_update}
\end{equation}
In other words, conditioning on $
(\theta_t, X_t, Y_{t+1})$, $\theta_{t+1}$ is independent of the dataset $\data_t$ and the environment $\environment$. Note that the incremental update rule in equation~\eqref{eq:incremental_update} is general: in particular, $\data_t$ could itself be recorded in $\theta_t$.  This would allow $\theta_{t+1}$ to depend on $\data_t$ in an arbitrary manner.  However, such an approach can be impractical when there is a high volume of data.  In particular, one may want to avoid sifting through a growing $\data_t$ at each time step. In many practical applications, it is desirable for the agent to update $\theta_{t+1}$ with fixed memory space and fixed per-step computational complexity, such as the standard SGD \cite{goodfellow2016deep} and Adam \cite{kingma2014adam} algorithms do.

We have the following result for the considered sequential prediction problem:

\begin{theorem}
\label{thm:info_retain}
For an agent with incremental update \eqref{eq:incremental_update}, for any time $t=0,1, \ldots, T-1$ and any $\epsilon \geq 0$, if 
\[
\textstyle \sum_{t'=t}^{T-1} \E \big [ \KL \big(\bar{P}_{t'+1} \, \big \| \, \hat{P}_{t'+1} \big) \big] \leq \epsilon,
\]
then we have 
\[
\I \left(Y_{t+1:T}; \theta_t  \, \middle | \, X_{t:T-1} \right)
\geq \I \left(Y_{t+1:T}; \data_t \, \middle | \, X_{t:T-1} \right) - \epsilon.  
\]

\end{theorem}
Please refer to Appendix~\ref{app:streaming} for the proof of Theorem~\ref{thm:info_retain}.
Notice that $\epsilon$ measures the performance loss of the agent; $\I \left(Y_{t+1:T}; \data_t \, \middle | \, X_{t:T-1} \right)$ is the conditional information in $\data_t$ about the joint distribution of $Y_{t+1:T}$; and similarly $\I \left(Y_{t+1:T}; \theta_t \, \middle | \, X_{t:T-1} \right)$ is the conditional information about $Y_{t+1:T}$ retained in $\theta_t$. Also notice that
\[
\I \left(Y_{t+1:T}; \data_t \, \middle | \, X_{t:T-1} \right) \geq 
\I \left(Y_{t+1:T}; \theta_t  \, \middle | \, X_{t:T-1} \right)
\]
always holds due to data processing inequality.
In other words, Theorem~\ref{thm:info_retain} states that to be $\epsilon$-near-optimal, an agent with incremental update must retain in $\theta_t$ all information in $\data_t$ about the joint distribution of $Y_{t+1: T}$, except $\epsilon$ nats\footnote{In this paper, we use natural logarithms in both KL-divergence and mutual information. Notice that $\epsilon$ nats $\approx$ $1.443 \epsilon$ bits.}. 

We conjecture that results similar to Theorem~\ref{thm:info_retain} also hold in broader classes of sequential decision problems, such as multi-armed bandit problems discussed in Section~\ref{sec:mab}, but we leave the formal analysis to future work.

%%%%%%%%%%%%%%%%%%%%%%%%%%%%%%%%%%%%%%%%%%%%%%%%%%%%%%%%%%%%%%%%%%%%%%%%%%%%%%%%%%%%%%%%%%%%%%%%%%%% Sequential decision
%%%%%%%%%%%%%%%%%%%%%%%%%%%%%%%%%%%%%%%%%%%%%%%%%%%%%%%%%%%%%%%%%%%%%%%%%%%%%%%%%%%%%%%%%%%%%%%%%%%%
\section{Multi-armed bandits}
\label{sec:mab}

In this section we turn our attention from passive observations, where an agent's predictions do not influence the stream of data $\data_t$, to active decisions, where an agent takes actions that influence future streams of data.
In this context, the ability to predict outcomes over multiple time steps can significantly impact agent performance.
In particular, we will show that the quality of predictions beyond marginals is essential to effectively balance the needs of exploration with exploitation \cite{Thompson1933}.
We begin with a problem formulation around the famous ``multi-armed bandit'', a simple model of sequential decisions that serves to highlight the key issues at play.
Next, we show that, in this context, even agents that make perfect marginal predictions for the outcome of each action may not make good decisions for the problem as a whole.
Finally, we show that agents that make good joint predictions over $\tau = O(K)$ time steps are sufficient to drive efficient exploration in a Bernoulli bandit with $K$ actions.

%%%%%%%%%%%%%%%%%%%%%%%%%%%%%%%%%%%%%%%%%%%%%%%%%%%%%%%%%%%%%%%%%%%%%%%%%%%%%%%%%%%%%%%%%%%%%%%%%%%% Formulation
\subsection{Problem formulation}

Consider a sequential decision problem with $K$ actions. The environment $\environment$ 
prescribes observation probabilities. During each time step $t = 0, 1, 2, \ldots$, the agent selects an action $A_t$ and observes an outcome $Y_{t+1}$ produced by the environment. Conditioned on the environment $\environment$ and action $A_t$, the next observation $Y_{t+1}$ is sampled from $\environment(\cdot |  A_t)$ independently across time. There is a real-valued reward function $r$ that encodes the agent's preferences over outcomes. The objective of the agent is to optimize the expected return over some long horizon $T$, 
\[ 
\E \big[ \textstyle \sum_{t=0}^{T-1} r(Y_{t+1}) \big], \]
where the expectation is taken over random outcomes, algorithmic randomness, and prior over $\environment$.

%%%%%%%%%%%%%%%%%%%%%%%%%%%%%%%%%%%%%%%%%%%%%%%%%%%%%%%%%%%%%%%%%%%%%%%%%%%%%%%%%%%%%%%%%%%%%%%%%%%% Marginal
\subsection{Marginal predictions are insufficient}
\label{sec:mab-marginal}

We will now show that, even if an agent can make perfect predictions for the outcomes of each action, these marginal predictions per action can be insufficient to drive efficient exploration.
The root cause is that, at any time step $t$, the rewards and observations at future time steps $t' > t$ are coupled through the unknown environment $\environment$.
Clearly, we can imagine problem instances where one action reveals the full structure of the environment, but does not itself produce a reward.
In this case, an agent that simply makes good marginal predictions on the reward of each action is clearly insufficient to take advantage of this informative action.
However, perhaps more interestingly, this coupling can occur even when rewards evolve independently per action given the environment.

% this coupling can occur even when, given knowledge of the environment, rewards actually evolve independently per action.

% Example with indistinguishable arms
To make this point clear, we will consider an extension of the coin tossing example from Section \ref{sec:example_coin}.
We consider an Bernoulli bandit with $K$ independent actions where the reward is simply taken to be the $Y_{t+1}$.
For the first $K-1$ actions the agents knows that the outcome is distributed $\mathrm{Ber(0.5)}$, a pure 50\% chance.
However, the final action produces a deterministic outcome of either $1$ or $0$, but it is equally likely to be of either kind.
Clearly, the optimal policy to maximize long term reward is to first select the final action to see if it is optimal and, based on that outcome, choose the arm that maximizes expected reward given full knowledge of $\environment$ for all future steps.
However, any agent that only matches marginal predictions cannot distinguish the final arm from any of the others.
As such, it will be impossible for the agent to do better than random chance over the $K$ arms. On average, it will require $\Theta(K)$ time steps to identify the optimal policy.

Just as in the coin flip example, the difference between the informative and uninformative actions is fully revealed when we look at $\tau^{\textrm{th}}$-order predictive distributions for any $\tau > 1$.
In general, it is only through examining the evolution of future outcomes beyond myopic marginals that allows an agent to optimize the long term returns.
In the next subsection, we will expand on this intuitive argument and relate the ability to predict future outcomes to agent performance.
In fact, we will establish a bound on performance that guarantees that if you can make approximate predictions $O(K)$ steps into the future, then it will be sufficient to drive efficient exploration.

%%%%%%%%%%%%%%%%%%%%%%%%%%%%%%%%%%%%%%%%%%%%%%%%%%%%%%%%%%%%%%%%%%%%%%%%%%%%%%%%%%%%%%%%%%%%%%%%%%%% Large tau
\subsection{Relating joint predictions to regret}
\label{sec:mab-regret}

As explained in the previous subsection, in a sequential decision problem, it is in general desirable to consider uncertainties jointly across actions in order to make effective decisions. Thus, to simplify exposition, we consider predictions for a vector of outcomes, $\Yb \in \Re^K$, with each entry corresponding to the outcome of an action. In this section, we will relate the quality of future predictions about $\Yb$ to agent performance on a Bernoulli bandit with correlated arms. 

Recall that in a $K$-armed Bernoulli bandit, the environment $\environment$ is identified by parameters $p = (p_1,\ldots,p_K)$, where $p_k \in [0, 1]$ is the expected reward of the $k^{\rm th}$ action. We make no assumptions about the prior $\Prob(p \in \cdot)$.  In particular, there could be dependencies, where an observed outcome from trying one action informs our beliefs about other actions. We define the history by time $t$ as  $H_t = \left( A_0, Y_1, \ldots, A_{t-1}, Y_t \right)$.

We use $\tilde{\Yb}_{1:\tau}$ to denote a sequence of $\tau$ vectors sampled from the environment $\environment$. Specifically, these $\tau$ vectors are conditionally independent given $\environment$. Each vector has dimension $K$, and the $k^{\rm th}$ component of each vector is conditionally independently sampled from $\mathrm{Ber}(p_k)$. On the other hand, consider an agent that can also generate a sequence of $K$-dimensional binary vectors at each time $t$. Let $\theta_t$ denote its parameters, and $\hat{\Yb}^t_{1:\tau}$ denote a sequence of $\tau$ binary vectors sampled from it. In this subsection, we establish that if an agent can produce good $\tau^{\rm th}$-order predictive distributions in the sense that
\[ \E \left[ \KL(\Prob(\tilde{\Yb}_{1:\tau} \in \cdot | H_t) \| \Prob(\hat{\Yb}^t_{1:\tau} \in \cdot | \theta_t)) \right] \leq \epsilon \quad \forall t \]
for some small $\epsilon \geq 0$, then based on this agent one can design bandit algorithms that achieve near-optimal performance.

Specifically, we consider an approximate version of Thompson sampling algorithm, which is described in Algorithm~\ref{alg:ats}. This algorithm proceeds as follows: at each time $t$, it first samples $\tau$ binary vectors $\hat{\Yb}^t_{1:\tau}$ based on the agent predictive distribution; then, it samples a vector $\hat{p}^t \in [0, 1]^K$ from the conditional distribution
$\Prob(p \in \cdot | \tilde{\Yb}_{1:\tau} = \hat{\Yb}^t_{1:\tau})$; finally, it chooses an action $A_t$ greedy to $\hat{p}^t$ and update the agent parameters based on new observations. Note that Algorithm~\ref{alg:ats} is general in the sense that it does not depend on the agent's uncertainty representation. Instead, it only requires that the agent can simulate hypothetical observations, sampled from a joint predictive distribution. Also note that Algorithm~\ref{alg:ats} reduces to the standard (exact) Thompson sampling algorithm when $\Prob(\hat{\Yb}_{1:\tau} \in \cdot | \theta_t) = \Prob(\tilde{\Yb}_{1:\tau} \in \cdot | H_t)$ and $\tau \rightarrow \infty$.
We use this algorithm to establish that an agent that performs well based on a particular loss function retains enough information to enable efficient exploration.\footnote{Note that $\min \argmax_k \hat{p}^t_k$ is well defined. Specifically, $ \argmax_k \hat{p}^t_k \subseteq \{1,\ldots,K\}$ is a set.}

% Note that Algorithm~\ref{alg:ats} is not meant to be a practical algorithm; instead, it offers means to establish that an agent that performs well based on a particular loss function retains enough information to enable efficient exploration. 

\begin{algorithm}
\caption{Approximate Thompson sampling}
\label{alg:ats}
\begin{algorithmic}
\STATE \textbf{Input:} prior over environment parameters $p$ 
\STATE \qquad \quad agent architecture 
\STATE \qquad \quad agent parameter initialization/update procedure
\STATE \textbf{Initialization:} compute parameters $\theta_0$ based on prior
\FOR{$t=0,1,2,\ldots$}
\STATE sample $\hat{\Yb}^t_{1:\tau} \sim \Prob(\hat{\Yb}_{1:\tau} \in \cdot | \theta_t)$
\STATE sample $\hat{p}^t$ from $\Prob(p \in \cdot | \tilde{\Yb}_{1:\tau} = \hat{\Yb}^t_{1:\tau})$
\STATE choose $A_t = \min \argmax_k \hat{p}^t_k$
\STATE compute $\theta_{t+1}$ based on $\theta_t$ and $(A_t, Y_{t+1})$.
\ENDFOR
\end{algorithmic}
\end{algorithm}

As is standard in bandit literature, we measure the performance of Algorithm~\ref{alg:ats} using (Bayes) cumulative regret, which is defined as
\begin{equation}
    \Regret(T) = \textstyle \sum_{t=0}^{T-1}\E \big[ p_{A^*} - r(Y_{t+1}) \big],
\end{equation}
where $A^* = \min \argmax_k p_k$ is one optimal action. Similarly, the expectation is over random outcomes, algorithmic randomness, and prior over $\environment$.
We can establish the following regret bound for Algorithm~\ref{alg:ats}:
\begin{theorem}
\label{thm:regret_bound}
For any integer $\tau \geq 1$ and any $\epsilon \in \Re_+$,
if at each time $t$, the agent with parameters $\theta_t$ can generate samples $\hat{\Yb}^t_{1:\tau}$ such that
$$\E[\KL(\Prob(\tilde{\Yb}_{1:\tau} \in \cdot | H_t)\| \Prob(\hat{\Yb}_{1:\tau}^t \in \cdot | \theta_t))] \leq \epsilon,$$
then under Algorithm~\ref{alg:ats}, we have
\begin{equation}
   \Regret(T) \leq 
   \sqrt{\frac{1}{2} K T \log K} + \left( \frac{K}{\sqrt{2\tau}} + \sqrt{2 \epsilon}\right) T.
\end{equation}
\end{theorem}
\noindent
Please refer to Section~\ref{sec:proof_sketch} for the proof sketch for Theorem~\ref{thm:regret_bound}, and the detailed proof is provided in Appendix~\ref{app:regret_bound}.
Specifically, the proof for this theorem uses information-ratio analysis similar to previous papers \cite{russo2016information, lu2021reinforcement}, and the notions of \emph{environment proxy} and \emph{learning target} developed in \citet{lu2021reinforcement} (see Section 3.3 of that paper). Specifically, we choose the environment proxy as $\tilde{\Yb}_{1:\tau}$ and the learning target as an action $\tilde{A} = \min \argmax_k \tilde{p}_k$, where $\tilde{p} = (\tilde{p}_1, \ldots, \tilde{p}_K)$ is an independent sample drawn from $\Prob(p \in \cdot | \tilde{\Yb}_{1:\tau})$. 
% Note that for sufficiently large $\tau$, $\tilde{A}$ is near-optimal; also 
Note that conditioning on the environment $\environment$,  $\tilde{\Yb}_{1:\tau}$ and $\tilde{A}$ are independent of the history $H_t$.

We now briefly discuss the regret bound in Theorem~\ref{thm:regret_bound}. First, note that if an agent can make good predictions $\tau \geq K/\epsilon$ steps into the future, then this regret bound reduces to
$O ( \sqrt{KT \log(K)} + \sqrt{K \epsilon} T )$, which is sufficient to ensure efficient exploration. Second, notice that this regret bound consists of three terms. The linear regret term $\sqrt{2 \epsilon} T$ is due to the expected KL-divergence loss of the agent. Specifically, if the agent makes a perfect prediction in the sense that 
$
\Prob(\hat{\Yb}_{1:\tau}^t \in \cdot | \theta_t) =
\Prob( \tilde{\Yb}_{1:\tau} \in \cdot | H_t)$ for all $t$, then this linear regret term will reduce to zero. On the other hand, another linear regret term $KT /\sqrt{2 \tau}$ is due to the fact that we choose $\tilde{A}$ as the learning target, which can be a sub-optimal action. It is obvious that as $\tau \rightarrow \infty$, $\tilde{A}$ will converge to $A^*$ and this linear regret term will reduce to zero.
Finally, the sublinear regret term $\sqrt{\frac{1}{2} K T \log K }$ is exactly the regret bound for the exact Thompson sampling algorithm \cite{russo2016information}. This is not surprising since when $\epsilon=0$ (i.e. $
\Prob(\hat{\Yb}_{1:\tau}^t \in \cdot | \theta_t) =
\Prob( \tilde{\Yb}_{1:\tau} \in \cdot | H_t)$) and $\tau \rightarrow \infty$, Algorithm~\ref{alg:ats} reduces to the exact Thompson sampling algorithm.

Finally, it is worth mentioning that we conjecture that we can develop results similar to Theorem~\ref{thm:regret_bound} for more practical algorithms and more general sequential decision problems.
Note that in Algorithm~\ref{alg:ats}, sampling $\hat{p}^t$ from $\Prob(p \in \cdot | \tilde{\Yb}_{1:\tau} = \hat{\Yb}^t_{1:\tau})$ can be computationally expensive. Instead, a computationally more efficient approach is to choose $\hat{p}^t$ as the sample mean of $\tilde{\Yb}_{1:\tau}$, i.e.
$\hat{p}^t = \frac{1}{\tau} \sum_{i=1}^{\tau} \hat{\Yb}^t_i$, where $\hat{\Yb}^t_i$ is the $i^{\rm th}$ vector in $\tilde{\Yb}_{1:\tau}$. We conjecture that one can derive a similar regret bound with this modification, but leave the analysis to future work. Moreover, we believe that similar approximate Thompson sampling algorithms and regret bounds can be developed for more general sequential decision problems, such as 
multi-armed bandits with unbounded noises and episodic reinforcement learning, but leave them to future work.

\subsection{Proof Sketch for Theorem~\ref{thm:regret_bound}}
\label{sec:proof_sketch}

We provide a proof sketch for Theorem~\ref{thm:regret_bound} in this subsection. First, notice that the expected per-step regret at time $t$ is
$\E [p_{A^*} - p_{A_t}]$, which can be decomposed as
\begin{equation}
\label{eq:regret_decomp_1}
\E[p_{A^*} -  p_{A_t}] = \E[p_{A^*} -  p_{\tilde{A}}] + \E[p_{\tilde{A}} - p_{A_t}].
\end{equation}
Recall that action $\tilde{A}$ is the learning target.
% The first term in the righthand side of equation~\eqref{eq:regret_decomp_1} is the performance loss of the learning target $\tilde{A}$ with respect to the optimal action $A^*$, while the second term is the performance loss of the chosen action $A_t$ with respect to $\tilde{A}$. 
We bound the two terms in the righthand side of equation~\eqref{eq:regret_decomp_1} separately. First, based on the fact that $p$ and $\tilde{p}$ are conditionally i.i.d given the environment proxy $\tilde{\Yb}_{1:\tau}$, we can show that
$$
\E[p_{A^*} -  p_{\tilde{A}}] \leq K/ \sqrt{2 \tau} .$$
To bound the second term $\E[p_{\tilde{A}} - p_{A_t}]$, we consider its conditional version $\E_t[p_{\tilde{A}} - p_{A_t}]$, where the subscript $t$ denotes conditioning on the history $H_t$. Using information-ratio analysis, we can prove that
\begin{multline}
    \E_t[p_{\tilde{A}} - p_{A_t}] \leq \sqrt{\frac{K}{2}\I_t(\tilde{A};{A}_t, \Yb_{{A}_{t}})} \nonumber \\
    + \big \|\Prob_t(\tilde{A} \in \cdot) - \Prob_t(A_t \in \cdot)\big \|_1.
\end{multline}
Using Pinsker's inequality, the data processing inequality, and the assumption on the expected KL-divergence in Theorem~\ref{thm:regret_bound}, we can bound that 
\[
\E \big [ \big \|\Prob_t(\tilde{A} \in \cdot) - \Prob_t(A_t \in \cdot)\big \|_1 \big ] \leq \sqrt{2 \epsilon}.
\]
% Finally, we bound the term $\sum_{t=0}^{T-1} \E \big [ \sqrt{\I_t(\tilde{A};{A}_t, \Yb_{{A}_{t}})} \big ]$. 
On the other hand, based on Cauchy–Schwarz inequality and the chain rule for mutual information, we have
\[
\sum_{t=0}^{T-1} \E \left [ \sqrt{\I_t(\tilde{A};{A}_t, \Yb_{{A}_{t}})} \right ] \leq \sqrt{T \I(\tilde{A}; H_T)}.
\]
Finally, note that $\I(\tilde{A}; H_T) \leq \H(\tilde{A}) \leq \log K$. Combining the above inequalities, we have proved Theorem~\ref{thm:regret_bound}.

\section{Related work}
\label{sec:related_work}

There is extensive literature focusing on evaluating marginal distributions (e.g. \citet{bishop2006pattern}, \citet{friedman2001elements}, \citet{lecun2015deep}), paired with evaluation on downstream tasks such as bandits \cite{riquelme2018deep}.
There is also a rich literature on developing agents for uncertainty estimation \citep{blundell2015weight, Gal2016Dropout, osband2015bootstrapped, welling2011bayesian}, and in most cases the developed agents are evaluated only based on marginal predictive distributions.

To the best of our knowledge, there is little literature that highlight the importance of joint predictive distributions. One exception is \citet{wang2021beyond}. However,  \citet{wang2021beyond} states that “joint log-likelihood scores [are] determined almost entirely by the marginal log-likelihood scores...in practice, they provide little new information beyond marginal likelihoods,” and proposes an alternative metric for joint distributions based on \emph{posterior predictive correlations} (PPCs).  Our paper, however, establishes that the KL-divergence for joint predictive distributions, which are equivalent to log-likelihood scores, ensure effective decisions, while marginals do not.  In particular, our results in Section~\ref{sec:mab} show that agents that attain small marginal KL-divergence might not retain enough information to enable efficient exploration, while agents that attain small joint KL-divergence do.  Moreover, the PPC metric proposed by \citet{wang2021beyond} is designed for regression problems, while cross-entropy loss can be applied more broadly, for example, to classification as well as regression problems. There is also a recent empirical work that compares agents based on their joint predictive distributions \citep{osband2021evaluating}.

There is also extensive literature on combinatorial decision-making \citep{papadimitriou1998combinatorial}, sequential prediction \citep{shalev2011online}, multi-armed bandits \citep{auer2002finite, lattimore2020bandit}, and Thompson sampling \citep{Thompson1933,chapelle2011empirical, russo2017tutorial}. There are also some recent work on approximate versions of Thompson sampling algorithms \citep{lu2017ensemble, phan2019thompson, zhang2019scalable, yu2020graphical}. In this paper, we propose and analyze a novel variant of Thompson sampling algorithm for multi-armed bandits. Our analysis uses an information-theoretic approach, similar to papers in this research line \citep{russo2016information, russo2018learning, lu2021reinforcement}. In addition, our analysis makes use of the notions of environment proxy and learning target developed in \citet{lu2021reinforcement}.

% Many algorithms have been proposed for various combinatorial and sequential decision problems. In contrast to most existing papers, this paper does not aim to develop new algorithms for combinatorial and sequential decision problems. Instead, our goal is to highlight and justify the importance of joint predictive distributions in these decision problems. 

%%%%%%%%%%%%%%%%%%%%%%%%%%%%%%%%%%%%%%%%%%%%%%%%%%%%%%%%%%%%%%%%%%%%%%%%%%%%%%%%%%%%%%%%%%%%%%%%%%%% CONCLUSION
%%%%%%%%%%%%%%%%%%%%%%%%%%%%%%%%%%%%%%%%%%%%%%%%%%%%%%%%%%%%%%%%%%%%%%%%%%%%%%%%%%%%%%%%%%%%%%%%%%%%
\section{Conclusion}
\label{sec:conclusion}

This paper aims to elucidate the importance of the joint predictive distributions for a broad class of decision problems, including the combinatorial and sequential decision problems. Specifically, we have shown that in a simple combinatorial decision problem (Section~\ref{sec:combinatorial}), a sequential prediction problem (Section~\ref{sec:sequential_prediction}), and a multi-armed bandit problem (Section~\ref{sec:mab}), accurate marginal predictions are insufficient to drive effective decisions. Instead, accurate joint predictive distributions are necessary for good performance. We also show that, in the multi-armed bandit problem, accurate joint predictive distributions are sufficient to enable near-optimal performance of a novel variant of Thompson sampling by establishing a new kind of regret bound via information-theoretic analysis (Theorem~\ref{thm:regret_bound}).  While we evaluate the accuracy of marginal and joint predictive distributions based on KL-divergence (Section~\ref{sec:evaluation_metric}), or equivalently, cross-entropy loss, our insights on marginal versus joint predictions should extend to other evaluation metrics.

\bibliographystyle{icml2022}
\bibliography{reference}

\newpage 
\appendix
\onecolumn
\vspace*{0.1cm}
\begin{center}
    {\LARGE \textbf{Appendices}}
\end{center}
\vspace{0.4cm}

%%%%%%%%%%%%%%%%%%%%%%%%%%%%%%%%%%%%%%%%%%%%%%%%%%%%%%%%%%%%%%%%%%%%%%%%%%%%%%%% Supervised learning
%%%%%%%%%%%%%%%%%%%%%%%%%%%%%%%%%%%%%%%%%%%%%%%%%%%%%%%%%%%%%%%%%%%%%%%%%%%%%%%%
\section{Proof for Proposition~\ref{prop:universality}}
\label{app:proof_for_universality}

\begin{proof}
From Pinsker's inequality, we have
\[
\delta \left( \overline{P}_{T+1:T+\tau} , \hat{P}_{T+1:T+\tau} \right) \leq \sqrt{\frac{1}{2} \KL \big ( \overline{P}_{T+1:T+\tau} \big \| 
\hat{P}_{T+1:T+\tau}
\big)},
\]
where $\delta$ is the total variation distance. Without loss of generality, we assume that the supports of $\overline{P}_{T+1:T+\tau}$ and 
$\hat{P}_{T+1:T+\tau}$ are countable. Then we have
$\delta \left( \overline{P}_{T+1:T+\tau} , \hat{P}_{T+1:T+\tau} \right) = \frac{1}{2} \left \| \overline{P}_{T+1:T+\tau} - \hat{P}_{T+1:T+\tau} \right \|_1$. Consequently, we have
\[
\left \| \overline{P}_{T+1:T+\tau} - \hat{P}_{T+1:T+\tau} \right \|_1
 \leq \sqrt{2 \KL \big ( \overline{P}_{T+1:T+\tau} \big \| 
\hat{P}_{T+1:T+\tau}
\big)}.
\]
Recall that $\hat{a} \in \mathcal{A}$ maximizes 
$$\sum_{y_{T+1:T+\tau}} \hat{P}_{T+1:T+\tau}(y_{T+1:T+\tau}) r(a, y_{T+1:T+\tau})
$$
over all $a \in \mathcal{A}$.
Similarly, we assume $a^* \in \mathcal{A}$ maximizes
$$\sum_{y_{T+1:T+\tau}} \overline{P}_{T+1:T+\tau}(y_{T+1:T+\tau}) r(a, y_{T+1:T+\tau}),
$$
over all $a \in \mathcal{A}$. Then we have
\begin{align}
   & \,\left | \sum_{y_{T+1:T+\tau}} \overline{P}_{T+1:T+\tau}(y_{T+1:T+\tau}) r(a^*, y_{T+1:T+\tau}) -
   \sum_{y_{T+1:T+\tau}} \hat{P}_{T+1:T+\tau}(y_{T+1:T+\tau}) r(a^*, y_{T+1:T+\tau}) \right| \nonumber \\
   =& \, 
   \left | \sum_{y_{T+1:T+\tau}} \left( \overline{P}_{T+1:T+\tau}(y_{T+1:T+\tau}) - \hat{P}_{T+1:T+\tau}(y_{T+1:T+\tau}) \right)  r(a^*, y_{T+1:T+\tau}) \right| \nonumber \\
   \leq & \, 
   \sum_{y_{T+1:T+\tau}} \left | \overline{P}_{T+1:T+\tau}(y_{T+1:T+\tau}) - \hat{P}_{T+1:T+\tau}(y_{T+1:T+\tau}) \right |   r(a^*, y_{T+1:T+\tau}) \nonumber \\
   \stackrel{(a)}{\leq} & \, 
   \sum_{y_{T+1:T+\tau}} \left | \overline{P}_{T+1:T+\tau}(y_{T+1:T+\tau}) - \hat{P}_{T+1:T+\tau}(y_{T+1:T+\tau}) \right |   \nonumber \\
   =& \, \left \| \overline{P}_{T+1:T+\tau} - \hat{P}_{T+1:T+\tau} \right \|_1
 \leq \sqrt{2 \KL \big ( \overline{P}_{T+1:T+\tau} \big \| 
\hat{P}_{T+1:T+\tau}
\big)}, \nonumber
\end{align}
where (a) follows from $r(a^*, y_{T+1:T+\tau}) \in [0, 1]$. Consequently, we have
\begin{align}
    \sum_{y_{T+1:T+\tau}} \overline{P}_{T+1:T+\tau}(y_{T+1:T+\tau}) r(a^*, y_{T+1:T+\tau}) \leq & \, \sum_{y_{T+1:T+\tau}} \hat{P}_{T+1:T+\tau}(y_{T+1:T+\tau}) r(a^*, y_{T+1:T+\tau}) \nonumber \\
    +& \, \sqrt{2 \KL \big ( \overline{P}_{T+1:T+\tau} \big \| 
\hat{P}_{T+1:T+\tau}
\big)} \nonumber \\
\leq & \, \sum_{y_{T+1:T+\tau}} \hat{P}_{T+1:T+\tau}(y_{T+1:T+\tau}) r(\hat{a}, y_{T+1:T+\tau}) \nonumber \\
    +& \, \sqrt{2 \KL \big ( \overline{P}_{T+1:T+\tau} \big \| 
\hat{P}_{T+1:T+\tau}
\big)}, \nonumber
\end{align}
where the second inequality follows from the fact that $\hat{a}$ is optimal under $\hat{P}_{T+1:T+\tau}$. Taking the expectations, we have
\[
\E \left[r(\hat{a}, y_{T+1:T+\tau}) \right] \geq \E \left[r(a^*, y_{T+1:T+\tau}) \right] - \E \left[ \sqrt{2 \KL \big ( \overline{P}_{T+1:T+\tau} \big \| 
\hat{P}_{T+1:T+\tau}
\big)}
\right].
\]
From Jensen's inequality, we have
\[
\E \left[r(a^*, y_{T+1:T+\tau}) \right] \geq \max_{a \in \mathcal{A}} \E \left[r(a, y_{T+1:T+\tau}) \right]
\]
and
\[
\E \left[ \sqrt{2 \KL \big ( \overline{P}_{T+1:T+\tau} \big \| 
\hat{P}_{T+1:T+\tau}
\big)}
\right] \leq
 \sqrt{2 \E \left[ \KL \big ( \overline{P}_{T+1:T+\tau} \big \| 
\hat{P}_{T+1:T+\tau}
\big) \right]} = \sqrt{2 \KL^\tau}.
\]
Consequently, we have
\[
\E \left[r(\hat{a}, y_{T+1:T+\tau}) \right] \geq \max_{a \in \mathcal{A}} \E \left[r(a, y_{T+1:T+\tau}) \right] - \sqrt{2 \KL^\tau}.
\]
This concludes the proof.
\end{proof}

\section{Proof for Theorem~\ref{thm:info_retain}} \label{app:streaming}

\subsection{Proof for Theorem~\ref{thm:info_retain}}

\begin{proof}
Notice that
\begin{align*}
&~ \E\left[\KL \big(\Prob \left(Y_{t+1:T} \in \cdot \middle | \data_t, X_{t:T-1}\right) \big \|
    \Prob(Y_{t+1:T} \in \cdot | \theta_t, X_{t:T-1})
    \big)\right] \\
\overset{(a)}{=} &~ \E\left[\KL \big(\Prob \left(Y_{t+1} \in \cdot \middle | \data_t, X_t\right) \big \|
    \Prob(Y_{t+1} \in \cdot | \theta_t, X_t)
    \big)\right] \\
&~ \quad + \E\left[\KL \big(\Prob \left(Y_{t+2:T} \in \cdot \middle | \data_{t+1}, X_{t+1:T-1} \right) \big \|
    \Prob(Y_{t+2:T} \in \cdot | \theta_t, X_t, Y_{t+1}, X_{t+1:T-1})
    \big)\right] \\
\overset{(b)}{\leq} &~ \E\left[\KL \big(\Prob \left(Y_{t+1} \in \cdot \middle | \data_t, X_t\right) \big \|
    \Prob(Y_{t+1} \in \cdot | \theta_t, X_t)
    \big)\right] \\
&~ \quad + \E\left[\KL \big(\Prob \left(Y_{t+2:T} \in \cdot \middle | \data_{t+1}, X_{t+1:T-1} \right) \big \|
    \Prob(Y_{t+2:T} \in \cdot | \theta_{t+1}, X_{t+1:T-1})
    \big)\right] \\
\leq &~ \dots \\
\leq &~ \E\left[\sum_{t'=t}^{T-1} \KL \big(\Prob \left(Y_{t'+1} \in \cdot \middle | \data_{t'}, X_{t'} \right) \big \|
    \Prob(Y_{t'+1} \in \cdot | \theta_{t'}, X_{t'} )
    \big)\right] \\
\overset{(c)}{\leq} &~ \E\left[\sum_{t'=t}^{T-1}
    \KL \big(\Prob \left(Y_{t'+1} \in \cdot \middle | \data_{t'}, X_{t'} \right) \big \|
    \Prob(\hat{Y}_{t'+1} \in \cdot | \theta_{t'}, X_{t'})
    \big)\right] \\
\stackrel{(d)}{=} &~
\E\left[\sum_{t'=t}^{T-1}
    \KL \big( \bar{P}_{t'+1} \big \|
    \hat{P}_{t'+1}
    \big)\right] \leq \epsilon, 
\end{align*}
where (a) follows from the chain rule of KL divergence, (b) and (c) follow from the fact that the conditional predictive distribution minimizes the KL divergence to the posterior (see Lemma~\ref{lemma:conditional-predictive}), and (d) follows from the definition. On the other hand, by definition, we have
\[
\I \left( Y_{t+1:T}; \data_t \, \middle | \, \theta_t, X_{t:T-1}  \right ) = \E\left[\KL \big(\Prob \left(Y_{t+1:T} \in \cdot \middle | \data_t, X_{t:T-1}\right) \big \|
    \Prob(Y_{t+1:T} \in \cdot | \theta_t, X_{t:T-1})
    \big)\right] \leq \epsilon. 
\]
Moreover, from the chain rule of mutual information, we have
\begin{align}
    \I \left( Y_{t+1:T}; \data_t , \theta_t \, \middle | \,  X_{t:T-1}  \right ) = & \, \I \left( Y_{t+1:T}; \theta_t \, \middle | \,  X_{t:T-1}  \right ) + \underbrace{\I \left( Y_{t+1:T}; \data_t   \, \middle | \, \theta_t,  X_{t:T-1}  \right )}_{\leq \epsilon} \nonumber \\
     = & \, \I \left( Y_{t+1:T}; \data_t \, \middle | \,  X_{t:T-1}  \right ) + \underbrace{\I \left( Y_{t+1:T}; \theta_t   \, \middle | \, \data_t,  X_{t:T-1}  \right )}_{=0}, \nonumber 
\end{align}
where $\I \left( Y_{t+1:T}; \theta_t   \, \middle | \, \data_t,  X_{t:T-1}  \right ) = 0$ since $ Y_{t+1:T}$ and $\theta_t$ are conditionally independent given $\data_t$ and $X_{t:T-1}$. Hence we have
\[
\I \left( Y_{t+1:T}; \data_t \, \middle | \,  X_{t:T-1}  \right ) \geq 
\I \left( Y_{t+1:T}; \theta_t \, \middle | \,  X_{t:T-1}  \right ) \geq \I \left( Y_{t+1:T}; \data_t \, \middle | \,  X_{t:T-1}  \right ) - \epsilon.
\]
This concludes the proof.
\end{proof}

\subsection{Lemma~\ref{lemma:conditional-predictive}}

\begin{lemma} \label{lemma:conditional-predictive}
Let the agent parameter $\theta_t \perp \environment | \data_t$. Then, the conditional predictive distribution $\Prob(Y_{t+1} \in \cdot | \theta_t, X_t)$ minimizes the expected KL divergence
\[ \E\left[ \KL\left( \Prob(Y_{t+1}\in \cdot | \data_t, X_t) \big\| \Prob(\hat{Y}_{t+1} \in \cdot | \theta_t, X_t) \right) \right] \]
over all possible predictive distribution $\Prob(\hat{Y}_{t+1}\in \cdot | \theta_t, X_t)$.
\end{lemma}
\begin{proof}
We have
\begin{align*}
& \E\left[ \KL\left( \Prob(Y_{t+1} \in \cdot | \data_t, X_t) \big\| \Prob(\hat{Y}_{t+1} \in \cdot | \theta_t, X_t) \right) \right] \\
&= \E\left[ \sum_y \Prob(Y_{t+1}=y | \data_t, X_t) \log \left( \frac{\Prob(Y_{t+1}=y | \data_t, X_t)}{\Prob(Y_{t+1}=y | \theta_t, X_t)} \cdot \frac{\Prob(Y_{t+1}=y | \theta_t, X_t)}{\Prob(\hat{Y}_{t+1}=y | \theta_t, X_t)} \right) \right] \\
&= \E\left[ \KL\left( \Prob(Y_{t+1} \in \cdot | \data_t, X_t) \big\| \Prob(Y_{t+1} \in \cdot | \theta_t, X_t) \right) \right] \\
&\qquad + \E\left[ \sum_y \Prob(Y_{t+1}=y | \data_t, X_t) \log  \frac{\Prob(Y_{t+1}=y | \theta_t, X_t)}{\Prob(\hat{Y}_{t+1}=y | \theta_t, X_t)} \right] \\
&= \E\left[ \KL\left( \Prob(Y_{t+1} \in \cdot | \data_t, X_t) \big\| \Prob(Y_{t+1} \in \cdot | \theta_t, X_t) \right) \right] \\
&\qquad + \E\left[ \sum_y \Prob(Y_{t+1}=y | \theta_t, X_t) \log  \frac{\Prob(Y_{t+1}=y | \theta_t, X_t)}{\Prob(\hat{Y}_{t+1}=y | \theta_t, X_t)} \right] \\
&= \E\left[ \KL\left( \Prob(Y_{t+1} \in \cdot | \data_t, X_t) \big\| \Prob(Y_{t+1} \in \cdot | \theta_t, X_t) \right) \right] \\
&\qquad + \E\left[ \KL\left( \Prob(Y_{t+1} \in \cdot | \theta_t, X_t) \big\| \Prob(\hat{Y}_{t+1} \in \cdot | \theta_t, X_t) \right) \right] \\
&\geq \E\left[ \KL\left( \Prob(Y_{t+1} \in \cdot | \data_t, X_t) \big\| \Prob(Y_{t+1} \in \cdot | \theta_t, X_t) \right) \right].
\end{align*}
\end{proof}
Similarly, we can prove that 
\begin{multline}
   \E\left[\KL \big(\Prob \left(Y_{t+2:T} \in \cdot \middle | \data_{t+1}, X_{t+1:T-1} \right) \big \|
    \Prob(Y_{t+2:T} \in \cdot | \theta_t, X_t, Y_{t+1}, X_{t+1:T-1})
    \big)\right]
   \\
    \leq \E\left[\KL \big(\Prob \left(Y_{t+2:T} \in \cdot \middle | \data_{t+1}, X_{t+1:T-1} \right) \big \|
    \Prob(Y_{t+2:T} \in \cdot | \theta_{t+1}, X_{t+1:T-1})
    \big)\right]. \nonumber
\end{multline}
To see it, notice that $\theta_{t+1} \sim \Prob(\theta_{t+1} \in \cdot | \theta_t, X_t, Y_{t+1}, t)$, and hence
$\Prob(Y_{t+2:T} \in \cdot | \theta_{t+1}, X_{t+1:T-1})$ is a particular choice of 
$\Prob(\hat{Y}_{t+2:T} \in \cdot | \theta_t, X_t, Y_{t+1}, X_{t+1:T-1})$. Consequently, using an analysis similar to Lemma~\ref{lemma:conditional-predictive}, we can prove the above inequality.

\section{Proof for Theorem~\ref{thm:regret_bound}}
\label{app:regret_bound}

\begin{proof}
We will now analyze the approximate TS algorithm.  For this purpose, we continue to use a proxy $\proxy = \tilde{\Yb}_{1:\tau}$ and target $\tilde{A} = \min \argmax_k \tilde{p}_k$.
Further, we define $\tilde{\Yb}^t_{1:\tau} \sim \Prob(\tilde{\Yb}_{1:\tau} \in \cdot | H_t)$, $\tilde{p}^t \sim \Prob(p \in \cdot | \tilde{\Yb}_{1:\tau} = \tilde{\Yb}^t_{1:\tau})$, and $\tilde{A}_t = \min \argmax_k \tilde{p}^t_k$.

Notice that
\begin{equation}
\label{eqn:regret decomposition}
\E[p_{A^*} -  p_{A_t}] = \E[p_{A^*} -  p_{\tilde{A}}] + \E[p_{\tilde{A}} - p_{A_t}],
\end{equation} 
and we first show that
$
\E[p_{A^*} -  p_{\tilde{A}}] \leq \frac{K}{2\tau}$. Notice that
\begin{align*}
\E[p_{A^*} - p_{\tilde{A}}]
=& \E[\E[p_{A^*} - p_{\tilde{A}} | \proxy]] 
= \E[\E[\tilde{p}_{\tilde{A}} - p_{\tilde{A}} | \proxy]] \\
\leq& \E[\E[\max_k (\tilde{p}_k - p_k) | \proxy]] 
\leq \sum_{k=1}^K \E[|\tilde{p}_k - p_k|] 
\stackrel{(a)}{\leq} \frac{K}{\sqrt{2\tau}}.
\end{align*}
Please refer to \cite{mattner2003mean} for the proof of inequality (a).

Now we are going to bound the second term $\E[p_{\tilde{A}} - p_{A_t}]$ on the RHS of Equation \eqref{eqn:regret decomposition} through bounding its conditional version $\E_t[p_{\tilde{A}} - p_{A_t}]$ in terms of the information gain $\I_t(\tilde{A};{A}_t, \Yb_{{A}_{t}})$, where the subscript $t$ in $\E_t$ and $\I_t$ denotes that the expectation and mutual information condition on the history $H_t$. In other words, $\E_t[p_{\tilde{A}} - p_{A_t}] = \E[p_{\tilde{A}} - p_{A_t} | H_t] $ and 
$\I_t(\tilde{A};{A}_t, \Yb_{{A}_{t}}) = \I(\tilde{A};{A}_t, \Yb_{{A}_{t}} | H_t = H_t) $. We first note that
\begin{align*}
%\I_t(\tilde{A}; p) - \I_t \left(\tilde{A}; p \big| {A}_t, \Yb^{t+1}_{{A}_{t}} \right) &= 
\I_t(\tilde{A};{A}_t, \Yb_{{A}_{t}})
&= \I_t(\tilde{A}; {A}_{t}) + \I_t(\tilde{A}; \Yb_{{A}_{t}}|A_t) \\
&= \I_t(\tilde{A}; \Yb_{{A}_{t}}|A_t) \\
&= \sum_a\Prob_t(A_t = a)\I_t(\tilde{A}; \Yb_{{A}_{t}}|A_t = a) \\
&= \sum_a\Prob_t(A_t = a)\I_t(\tilde{A}; \Yb_{a}) \\
&= \sum_a\Prob_t(A_t = a)\left(\sum_{\tilde{a}}\Prob_t(\tilde{A} = \tilde{a})\KL\left(\Prob_t(\Yb_{a} \mid \tilde{A} = \tilde{a}) \| \Prob_t(\Yb_{a})\right)  \right)\\
&= \sum_{a,\tilde{a}}\Prob_t(A_t = a)\Prob_t(\tilde{A} = \tilde{a})\KL\left(\Prob_t(\Yb_{a} \mid \tilde{A} = \tilde{a}) \| \Prob_t(\Yb_{a})\right)  
\end{align*}
where the first inequality uses the chain rule of mutual information; the second inequality follows from that $A_t$ and $\tilde{A}$ are conditionally independent; the fourth inequality uses that $A_t$ is conditionally jointly independent of $\tilde{A}$ and $\Yb=(\Yb_1,\ldots,\Yb_K)$; the fifth inequality follows from the KL divergence form of mutual information.

Now we are ready to bound $\E_t[p_{\tilde{A}} - p_{A_t}]$ in terms of $\I_t(\tilde{A};{A}_t, \Yb_{{A}_{t}})$.
\begin{align*}
    \E_t[p_{\tilde{A}} -  p_{A_t}] 
    =& \E_t[\Yb_{\tilde{A}} -  \Yb_{A_t}] \\
    =& \sum_a\Prob_t(\tilde{A} = a)\E_t[\Yb_{a}|\tilde{A} = a] -\sum_a\Prob_t(A_t = a)\E_t[\Yb_{a}|A_t = a] \\
    =& \sum_a\Prob_t(\tilde{A} = a)\E_t[\Yb_{a}|\tilde{A} = a] -\sum_a\Prob_t(A_t = a)\E_t[\Yb_{a}] \\
    =& \sum_a \sqrt{\Prob_t(A_t = a)\Prob_t(\tilde{A} = a)}\left(\E_t[\Yb_{a}|\tilde{A} = a] -\E_t[\Yb_{a}]\right) \\
    &+
    \sum_a\left(\sqrt{\Prob_t(\tilde{A} = a)}-\sqrt{\Prob_t(A_t = a)}\right)\left(\sqrt{\Prob_t(\tilde{A} = a)}\E_t[\Yb_{a}|\tilde{A} = a]
    +\sqrt{\Prob_t(A_t = a)}\E_t[\Yb_{a}]\right) \\
    \leq& \sqrt{K\sum_a \Prob_t(A_t = a)\Prob_t(\tilde{A} = a)\left(\E_t[\Yb_{a}|\tilde{A} = a] -\E_t[\Yb_{a}]\right)^2} \\
    &+ \sum_a \left|\sqrt{\Prob_t(\tilde{A}=a)}
    -\sqrt{\Prob_t(A_t=a)}\right|
    \left(\sqrt{\Prob_t(\tilde{A}=a)} + \sqrt{\Prob_t(A_t=a)}\right)\\
    %=& \sqrt{K\sum_a \Prob_t(A_t = a)\Prob_t(\tilde{A} = a)\left(\E_t[p_{a}|\tilde{A} = a] -\E_t[p_{a}]\right)^2}+ \sum_a \left|{\Prob_t(\tilde{A}=a)} - {\Prob_t(A_t=a)}\right|\\
    %\sum_a \left|\Prob_t(\tilde{A}_t=a) - \Prob_t(A_t=a)\right|\\ 
    \leq& \sqrt{K\sum_{a,\tilde{a}} \Prob_t(A_t = a)\Prob_t(\tilde{A} = \tilde{a})\left(\E_t[\Yb_{a}|\tilde{A} = \tilde{a}] -\E_t[\Yb_{a}]\right)^2} 
    + \sum_a \left|\Prob_t(\tilde{A}=a) - \Prob_t(A_t=a)\right|\\ 
    \leq& \sqrt{\frac{K}{2}\sum_{a,\tilde{a}} \Prob_t(A_t = a)\Prob_t(\tilde{A} = \tilde{a})\KL\left(\Prob_t(\Yb_{a} \mid \tilde{A} = \tilde{a}) \| \Prob_t(\Yb_{a})\right)} 
    + \sum_a\left|\Prob_t(\tilde{A}_t=a) - \Prob_t(A_t=a)\right|\\
    =& \sqrt{\frac{K}{2}\I_t(\tilde{A};{A}_t, \Yb_{{A}_{t}})} + \sum_a\left|\Prob_t(\tilde{A}_t=a) - \Prob_t(A_t=a)\right|
\end{align*}
where the third inequality uses that $A_t$ is conditionally independent of $\Yb=(\Yb_1,\ldots,\Yb_K)$; the first inequality follows from Cauchy-Schwarz inequality and $\Yb_a\in\{0,1\}$; the third inequality uses Pinsker's inequality and $\Prob_t(\tilde{A}=a)=\Prob_t(\tilde{A}_t=a)$. 
Hence,
\begin{align*}
    \E[p_{\tilde{A}} - p_{A_t}] &= \E[\E_t[p_{\tilde{A}} - p_{A_t}]] \\
    &\leq \E\left[\sqrt{\frac{K}{2}\I_t(\tilde{A};{A}_t, \Yb_{{A}_{t}})}\right] + \E\left[\sum_a\left|\Prob_t(\tilde{A}_t=a) - \Prob_t(A_t=a)\right|\right]\\
    &\leq \sqrt{\frac{K}{2}\E\left[\I_t(\tilde{A};{A}_t, \Yb_{{A}_{t}})\right]} +  \E\left[\sqrt{2 \KL(\Prob(\tilde{p}^t \in \cdot | H_t)\| \Prob(\hat{p}^t \in \cdot | H_t))}\right] \\
    &\leq \sqrt{\frac{K}{2}\E\left[\I_t(\tilde{A};{A}_t, \Yb_{{A}_{t}})\right]} + \E\left[\sqrt{2 \KL(\Prob(\tilde{\Yb}_{1:\tau} \in \cdot | H_t)\| \Prob(\hat{\Yb}^t_{1:\tau} \in \cdot | H_t))}\right] \\
    &\leq \sqrt{\frac{K}{2}\E\left[\I_t(\tilde{A};{A}_t, \Yb_{{A}_{t}})\right]} + \sqrt{2 \epsilon}
\end{align*}
where the second inequalities follows from Jensen's inequality and Pinsker's inequality and the third inequality uses the information processing inequality,
and thus
\begin{align*}
    \E[p_{A^*} -  p_{A_t}] = \E[p_{A^*} -  p_{\tilde{A}}] + \E[p_{\tilde{A}} - p_{A_t}]  \leq  \frac{K}{\sqrt{2\tau}} + \sqrt{\frac{K}{2} \E\left[\I_t(\tilde{A};{A}_t, \Yb_{{A}_{t}})\right]} + \sqrt{2 \epsilon}.
\end{align*}

Using an analysis similar to  \citet{lu2021reinforcement},
\begin{align*}
\Regret(T) 
=& \sum_{t=0}^{T-1} \E[p_{A_*} - p_{A_t}] \\
\leq& \sum_{t=0}^{T-1} \left(\frac{K}{\sqrt{2\tau}} + \sqrt{\frac{K}{2} \E\left[\I_t(\tilde{A};{A}_t, \Yb_{{A}_{t}})\right]} + \sqrt{2 \epsilon}\right) \\
%=& \sum_{t=0}^{T-1}  \sqrt{\frac{K}{2} \left(\I(\tilde{A}; p | H_t) - \I(\tilde{A}; p | H_{t+1}) + \sqrt{2 \epsilon} \log K \right)} + \left(\frac{K}{\sqrt{2\tau}} + \sqrt{2\epsilon}\right) T \\
\leq& \sqrt{\frac{K}{2} T}  \sqrt{\sum_{t=0}^{T-1}\E\left[\I_t(\tilde{A};{A}_t, \Yb_{{A}_{t}})\right] } + \left(\frac{K}{\sqrt{2\tau}} + \sqrt{2\epsilon}\right) T\\
\leq& \sqrt{\frac{K}{2} T}  \sqrt{\I(\tilde{A};A_0, \Yb_{A_0},\ldots, A_{T-1}, \Yb_{A_{T-1}})} + \left(\frac{K}{\sqrt{2\tau}} + \sqrt{2\epsilon}\right) T\\
=& \sqrt{\frac{K}{2} T}  \sqrt{\H(\tilde{A}) - \H(\tilde{A} \mid A_0, \Yb_{A_0},\ldots, A_{T-1}, \Yb_{A_{T-1}})} + \left(\frac{K}{\sqrt{2\tau}} + \sqrt{2\epsilon}\right) T\\
\leq& \sqrt{\frac{K}{2} T\log(K)}  + \left(\frac{K}{\sqrt{2\tau}} + \sqrt{2\epsilon}\right) T
\end{align*}
where the second inequality uses Cauchy–Schwarz inequality; the third inequality follows from the chain rule for mutual information; the last inequality follows from the non-negativity of entropy. This completes the proof.
\end{proof}

\end{document}